\documentclass[10pt,twocolumn,letterpaper]{article}
\usepackage[pagenumbers]{iccv} 

\usepackage{graphicx}
\usepackage{amsmath}
\usepackage{amssymb}
\usepackage{booktabs}
\usepackage{mathtools}
\usepackage{amsthm}
\usepackage{multirow}
\usepackage{comment}
\usepackage{pifont}

\theoremstyle{plain}
\newtheorem{theorem}{Theorem}[section]

\newtheorem{lemma}[theorem]{Lemma}

\theoremstyle{definition}
\newtheorem{definition}[theorem]{Definition}

\theoremstyle{remark}

\newcommand\R{\mathbb{R}}

\newcommand{\damien}[1]{\textbf{\color{blue}[D: #1]}}

\definecolor{iccvblue}{rgb}{0.21,0.49,0.74}
\usepackage[pagebackref,breaklinks,colorlinks,allcolors=iccvblue]{hyperref}

\def\paperID{15230}
\def\confName{ICCV}
\def\confYear{2025}

\title{Leaner Transformers: More Heads, Less Depth}
\author{
Hemanth Saratchandran\textsuperscript{1} \\
{\tt\small hemanth.saratchandran@adelaide.edu.au}
\and
Damien Teney\textsuperscript{2} \\
{\tt\small damien.teney@idiap.ch}
\and
Simon Lucey\textsuperscript{1} \\
{\tt\small simon.lucey@adelaide.edu.au}
\\[1ex]
\textsuperscript{1}Australian Institute for Machine Learning, University of Adelaide \\
\textsuperscript{2}Idiap Research Institute
}

\begin{document}
\maketitle
\begin{abstract}
Transformers have reshaped machine learning by utilizing attention mechanisms to capture complex patterns in large datasets, leading to significant improvements in performance. This success has contributed to the belief that ``bigger means better'', leading to ever-increasing model sizes.
This paper challenge this ideology by showing that many existing transformers might be unnecessarily oversized.
We discover a theoretical principle that redefines the role of multi-head attention.
An important benefit of the multiple heads is in improving the conditioning of the attention block.
We exploit this theoretical insight and redesign popular architectures with an increased number of heads.
The improvement in the conditioning proves so significant in practice that model depth can be decreased,
reducing the parameter count by up to 30-50\%
while maintaining accuracy.
We obtain consistent benefits across a variety of transformer-based architectures of various scales, on tasks in computer vision (ImageNet-1k) as well as language and sequence modeling (GLUE benchmark, TinyStories, and the Long-Range Arena benchmark).

\end{abstract}
    
\section{Introduction}\label{sec:intro}

\begin{figure}[t]
    \centering
    ~~~~~\textcolor[HTML]{FFA301}{\raisebox{-0.03em}{\scalebox{0.9}[0.9]{\ding{108}}}}~
    \footnotesize Ours (more heads, fewer layers)~~~
    \textcolor[HTML]{4D38A3}{\raisebox{-0.03em}{\scalebox{0.9}[0.9]{\ding{108}}}}~
    \footnotesize Original
    \\\vspace{2pt}
    \includegraphics[width=0.75\linewidth]{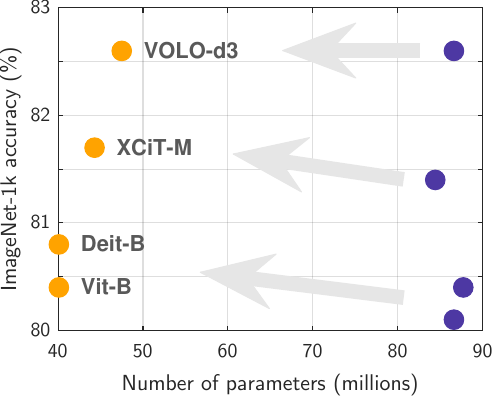}
    \vspace{-5pt}
    \caption{We redesign popular transformers models with an increased number of heads,
    using the theoretical insight that 
    multi-head attention contributes to improving the conditioning of attention blocks.
    The benefits are so significant that we can 
    reduce model depth while
    maintaining or improving accuracy, using about 50\% fewer parameters.}
    \label{fig:front_fig}
\end{figure}

Transformers~\cite{vaswani2017attention} have become the dominant architecture across a wide range of fields, including natural language processing (NLP)~\cite{vaswani2017attention, devlin2018bert, zhuang2021robustly, zhen2022cosformer}, computer vision~\cite{dosovitskiy2020image, carion2020end, liu2021swin, touvron2021training}, and robotics~\cite{fu2024drive, maiti2023transfusion, salzmann2020trajectron++}. At the heart of their success lies the attention mechanism, which dynamically assigns relevance scores to input elements, enabling the model to generate highly contextualized outputs. This ability allows transformers to capture complex dependencies
in data more effectively than traditional architectures.

As transformers continue to scale, the prevailing belief is that heavy overparameterization is necessary for strong performance. A standard decoder-only transformer increases capacity through three primary means: (1) expanding the number of attention heads, (2) widening the feedforward layers, and (3) deepening the network by adding more layers. However, no well-established guidelines exist for balancing these components to achieve optimal performance. Extensive research has explored the role of width and depth in improving optimization for convolutional and feedforward networks~\cite{agarwal2021deep, arora2018optimization, zhou2018understanding, kabkab2016size, li2018tighter, liu2022loss, jacot2018neural}.
For transformers however, our understanding of the trade-offs between width and depth remains incomplete~\cite{levine2020limits,levine2020depth,petty2023impact,sanford2023representational}.

In this paper, we challenge the conventional approach to transformer design and ask whether these models are structured optimally. We introduce a theoretical principle that offers a new perspective on the role of multi-head attention, demonstrating that it inherently improves the conditioning of attention layers.
This produces a matrix with a low condition number, which is the ratio of a matrix's largest to smallest singular values. This quantifies its stability: a high condition number indicates ill-conditioning, which can hinder convergence of gradient-based optimization~\cite{nocedal1999numerical}.
We theoretically show that using multiple heads lowers the condition number of attention layers and therefore facilitates the optimization of transformers.

We verify empirically that increasing the number of attention heads in transformers significantly improves the condition number of the attention block. We then use these insights to guide the design of transformer models, focusing on trade-offs with depth, one of the main choice in architecture design.
We find empirically that transformers can often be redesigned with more attention heads and fewer layers while maintaining both optimization stability and accuracy.
Since each layer corresponds to a large number of parameters,
trading additional heads for fewer layers
enables a substantial reduction
in model size without compromising performance.


We validate our findings by modifying and re-training
a range of existing models for vision and NLP tasks.
We show that attention heads can be consistently traded for depth, resulting in more parameter-efficient architectures without sacrificing performance (see \cref{fig:front_fig}).
While we lack a full theoretical explanation for this trade-off, our results raise important questions. Are transformers unnecessarily over-parameterized?
Are other trade-offs possible by improving the conditioning of existing architectures?
These results open multiple opportunities for future empirical and theoretical work.

\noindent
Our contributions are summarized as follows.
\begin{enumerate}[itemsep=1pt,topsep=0.5pt]
\item A theoretical framework offering a new perspective on multi-head attention, indicating that one of its core functions is to better condition the attention block.
\item An empirical design principle for transformers derived from our theoretical insights,
suggesting that model depth can be traded for additional heads to reduce parameter count without compromising accuracy.
\item A comprehensive empirical validation of downstream benefits for a variety of existing models on  standard vision and NLP tasks: image classification with ImageNet-1k~\cite{steiner2021train}, language modeling with TinyStories~\cite{eldan2023tinystories} and GLUE benchmark \cite{wang2018glue}, and long-context reasoning with the LRA benchmark~\cite{tay2021long}.
\end{enumerate}

\section{Related Work}\label{sec:rel_work}
\paragraph{Efficient attention-based architectures.}
Numerous approaches have been proposed to enhance the efficiency and effectiveness of transformers, particularly by reducing the computational complexity of the attention layer. DeiT (Data-Efficient Image Transformer)~\cite{touvron2021training} improves training efficiency by leveraging distillation tokens, enabling strong performance with significantly fewer data requirements. XCiT (Cross-Covariance Image Transformer)~\cite{ali2021xcit} introduces a novel attention mechanism that operates on spatial feature cross-covariances, improving feature interactions while substantially reducing computational overhead. VOLO (Vision Outlooker)~\cite{yuan2022volo} incorporates outlook attention, which efficiently captures long-range dependencies, outperforming traditional vision transformers (ViTs) while maintaining computational efficiency. Nyströmformer~\cite{xiong2021nystromformer} tackles the quadratic complexity of self-attention using a Nyström-based approximation, reducing it to near-linear time while preserving key attention properties.
Other efficient transformer variants have further addressed attention-related bottlenecks. Linformer~\cite{wang2020linformer} approximates self-attention with low-rank projections, achieving linear complexity by compressing the sequence length dimension. Performer~\cite{choromanski2020rethinking} employs kernelized attention with random feature projections, enabling scalable attention with linear time complexity. Reformer~\cite{kitaev2020reformer} utilizes locality-sensitive hashing to significantly reduce memory and computational costs, making attention efficient even for long sequences.

We take a different approach, exploring whether the inherent complexity of transformers can be reduced to create more compact models that maintain strong performance. 
Our insights on conditioning are orthogonal to the above methods
and we demonstrate benefits on several of the aforementioned architectures (ViTs, Nystr\"omformers).



\paragraph{Network width and depth.}
A vast literature has explored the roles of width and depth~\cite{lu2017expressive,poole2016exponential,vardi2022width}
and their interplay with gradient-based optimization.
For example, Liu et al.~\cite{liu2021swin} demonstrated that increasing the width of multi-layer perceptrons (MLPs) enhances the conditioning of their neural tangent kernel (NTK)~\cite{jacot2018neural}, leading to more effective optimization. Arora et al.~\cite{arora2018optimization} showed that, in linear MLPs, depth serves as a preconditioner for stochastic gradient descent, improving optimization as depth increases. Similarly, Agarwal et al.~\cite{agarwal2021deep} found that depth enhances the conditioning of non-linear MLPs, provided that activations are properly normalized, thereby facilitating better convergence with gradient-based algorithms.

The above studies underscore the importance of both width and depth in achieving good optimization for MLPs.
A similar theoretical understanding for transformers is lacking~\cite{levine2020limits,levine2020depth,petty2023impact,sanford2023representational}
and our work helps fill this gap. 
We also reveal a crucial role of the multi-head attention in the optimization of transformers
and explore its empirical relationship with model depth.


\section{Theoretical Findings}\label{sec:theory}

\subsection{Preliminaries}\label{subsec:prelims}

\paragraph{Transformers.}
We first briefly review the the transformer architecture~\cite{vaswani2017attention,dosovitskiy2020image}.
A transformer is composed of stacked layers, also known as ``transformer blocks''.
Each layer is formally represented as a mapping 
$\mathbf{T}: \mathbb{R}^{N \times D} \rightarrow \mathbb{R}^{N \times D}$
defined by the expression
    $\mathbf{T}(X) = \mathbf{F}(\mathbf{A}(X) + X)$.
The component \( \mathbf{F} \) denotes a feedforward multi-layer perceptron (MLP, typically with one hidden layer and a residual connection), and \( \mathbf{A} \) represents the self-attention mechanism.

The self-attention mechanism \( \mathbf{A} \) uses three learnable matrices,
the query (\( Q \)), key (\( K \)), and value (\( V \)) matrices.
Given an input sequence \( X \in \mathbb{R}^{N \times D} \),
the matrices are first applied as follows:
\( q\!=\!QX \), \( k\!=\!KX \), \( v\!=\!VX \),
where
\( Q, K \in \mathbb{R}^{D \times d} \)
\;and
\( V \in \mathbb{R}^{D \times M} \).
These are then combined to produce the output of the self-attention head as follows:
    $\mathbf{A}(X):=\operatorname{\mathbf{softmax}}(\,q\,k^T)\,v$,
where the softmax is applied row-wise. In this paper, whenever we speak of an attention matrix we will mean the matrix
$\operatorname{\mathbf{softmax}}(\,q\,k^T)\,v$.
Multiple parallel attention heads $\mathbf{A}_i$ are typically used ($1 \leq i \leq h$), each of dimension \( N \times \frac{D}{h} \).
Their outputs are concatenated as
    $[\mathbf{A}_1, \cdots, \mathbf{A}_h]$,
which is then fed into the MLP.
Additional normalizations and residual connections are often interleaved depending on the model's specific details.

\paragraph{Condition number.}
The condition number of a matrix is the ratio of its largest to smallest singular values. In gradient-based optimization of linear and non-linear systems, the condition number serves as a quantitative measure of how well the optimizer will converge. Lower values indicate a more stable and efficient convergence.
Conversely, a matrix is said to be \textbf{ill-conditioned} if the condition number is high.
Ill-conditioned matrices in non-linear systems lead to difficulties for gradient descent to converge~\cite{nocedal1999numerical}.

\begin{definition}\label{defn:condition_num}
  The \textbf{condition number} of a full-rank, \( n \times m \) matrix \( A \) is defined as
    $\kappa(A)\!:=\!\sigma_1(A) \;/\; \sigma_k(A)$, with the singular values
    \( \sigma_1(A) \geq \dots \geq \sigma_k(A) \) and \( k=\min(m, n)\).
\end{definition}
\noindent
Since \( A \) is of full rank, all singular values are positive and the condition number is thus well defined. And since $\sigma_1(A) \geq \sigma_k(A)$, the condition number satisfies $\kappa(A) \geq 1$. 


\subsection{Main Theoretical Result}
\label{subsec:main_theorem}

Our main finding states that multi-head attention
has the implicit effect of conditioning the self-attention block within a transformer layer,
which leads to attention matrices (\( A_i \)) with a low condition number.
This in turn facilitates the optimization of transformers by gradient descent.

\setlength{\abovedisplayskip}{6pt}
\setlength{\belowdisplayskip}{6pt}

\begin{theorem}\label{thm:main}
Let $\mathbf{A}_i \in \R^{N\times \frac{D}{h}}$ be i.i.d Gaussian random variables ($1 \leq i \leq h$).
We define the multi-head matrix block
 $\mathbf{A} = [\mathbf{A}_1, \cdots, \mathbf{A}_h]$
of dimension $N \times D$ and assume 
$D >> N$. Then, the condition number
\begin{equation}
    \kappa(\mathbf{A}) \approx 1.
\end{equation}
Moreover,
if we fix the dimension of the attention heads $d > 0$ such that
$\mathbf{A}_i \in \R^{N\times d}$, we have:
\begin{equation}\label{eqn:h_infty}
 \kappa(\mathbf{A}) \rightarrow 1 \text{ ~~as~~ } 
 h \rightarrow \infty.
\end{equation}
\end{theorem}

\noindent
To prove \cref{thm:main} we will need the following lemma.
\begin{lemma}\label{lem;random_draws}
Let $X$ be a matrix in $\R^{m\times n}$ with $n >> m$ whose entries are i.i.d drawn from a Gaussian distribution. Then $X$ is full rank with probability 1. 
\end{lemma}

\noindent
The proof of Lemma \ref{lem;random_draws} is given in \cref{app:theory}.

\begin{proof}[Proof of \cref{thm:main}]

The proof will proceed by using some well known facts about random matrices, see \cite{vershynin2018high} for proofs. Firstly given a random Gaussian matrix $X$ of full rank and size $m \times n$ with $n >> m$ we have that the minimum singular value, $\sigma_m(X)$, and maximum singular value of $X$, $\sigma_1(X)$, satisfy
\begin{equation}\label{supp:eqn:sing_vals_asymptotics}
    \sigma_m(X) \approx \sqrt{n} - \sqrt{m} \text{ and }
    \sigma_m(X) \approx \sqrt{n} + \sqrt{m}.
\end{equation}

\noindent
We start by proving the first part of the theorem. Observe that, by assumption, the multi-head matrix 
\begin{equation}
    \mathbf{A} = [\mathbf{A}_1, \cdots, \mathbf{A}_h]  
\end{equation}
has shape $N  \times D$ where $D >> N$. Furthermore, since each $\mathbf{A}_1$ is drawn i.i.d.\ from a Gaussian distribution, we have from Lemma~\ref{lem;random_draws} that $\mathbf{A}$ has full rank which is $N$. Then, applying Eq.~\eqref{supp:eqn:sing_vals_asymptotics} we have that 
\begin{equation}
 \sigma_N(\mathbf{A}) \approx \sqrt{D} - \sqrt{N}   
 \text{ and } 
 \sigma_1(\mathbf{A}) \approx \sqrt{D} + \sqrt{N}.   
\end{equation}
By the definition of the condition number, we then find that
\begin{equation}
\kappa(\mathbf{A}) := \frac{\sigma_1(\mathbf{A})}{\sigma_N(\mathbf{A})} \approx 
\frac{\sqrt{D} + \sqrt{N}}{\sqrt{D} - \sqrt{N}}.
\end{equation}
Since $D >> N$, we have 
$\frac{\sqrt{N}}{\sqrt{D} - \sqrt{N}} \approx 0$ and 
\begin{align}
\frac{\sqrt{D}}{\sqrt{D} - \sqrt{N}} &= 
\frac{\sqrt{D} - \sqrt{N}}{\sqrt{D} - \sqrt{N}} + 
\frac{\sqrt{N}}{\sqrt{D} - \sqrt{N}}     \\
&\approx \frac{\sqrt{D} - \sqrt{N}}{\sqrt{D} - \sqrt{N}} \\
&= 1
\end{align}
This then implies that
\begin{align}
\kappa(\mathbf{A}) &\approx 
\frac{\sqrt{D} +  \sqrt{N}}{\sqrt{D} - \sqrt{N}} \\
&= 
\frac{\sqrt{D}}{\sqrt{D} - \sqrt{N}} + 
\frac{\sqrt{N}}{\sqrt{D} - \sqrt{N}} \\
&\approx 
1
\end{align}
which proves the first part of the theorem.
\vspace{6pt}

\noindent
To prove the second part of the theorem, observe that if $h \rightarrow \infty$ then, using Eq.~\eqref{supp:eqn:sing_vals_asymptotics}, the condition number
is given by
\begin{equation}
\kappa(\mathbf{A}) \approx \frac{\sqrt{dh} + \sqrt{N}}{\sqrt{dh} - \sqrt{N}} \rightarrow 1 \text{ as } 
h \rightarrow \infty.\vspace{-5pt}
\end{equation}
\end{proof}
\vspace{-15pt}
The theorem highlights that, while an individual attention matrix
$\mathbf{A}_i$ of dimension $N \times \frac{D}{h}$
may not be well-conditioned, the \emph{concatenation} of multiple such matrices improves their overall condition number. This insight offers a new perspective on multi-head attention: it functions as an implicit conditioner, enhancing the conditioning of each attention block within a transformer.

\paragraph{Observation.} We observe that in \cref{eqn:h_infty} we could have also let $d$ go to infinity and the same proof shows that the matrix $\mathbf{A}$ would have condition number going to $1$. However, observe that, when $d$ is fixed, each attention head computes an $N\!\times\!d\,$ projection independently. With $h$~heads, these computations can be parallelized, allowing efficient scaling. In contrast, increasing~$d$ while keeping~$h$ fixed enlarges each head's computation, leading to slower training due to reduced parallelism. Therefore in this paper, we will focus on lowering the condition number of $\mathbf{A}$ by increasing the number of heads.


\subsection{Trading Depth for Heads}\label{subsec:heads_depth}


We demonstrated in \cref{thm:main} that additional heads
improve the conditioning of an attention layer.
We now examine how this can translate into tangible performance gains.

\paragraph{Conditioning in MLPs.}
The existing literature provides theoretical support for improved performance
of MLPs with better-conditioned weight matrices trained with gradient descent.
\citet{liu2022loss} used the Neural Tangent Kernel (NTK) framework~\cite{jacot2018neural}
to show that increasing network width reduces the NTK's condition number, thereby enhancing convergence.
As MLPs widen, their weight matrices enter the regime described in \cref{thm:main} where the condition number approaches $1$. By direct application of the chain rule, this implies that the improved conditioning of the weight matrices leads to a better-conditioned NTK. Complementary studies \cite{agarwal2021deep,arora2018optimization} reveal that increasing depth also helps conditioning for gradient-based optimizers. Together, these results underscore the dual importance of both width and depth in the optimization of MLPs.


\paragraph{What about transformers?}
Each transformer layer consists of a multi-head attention and an MLP.
Transformers employ wide MLPs, typically 2$\times$ to 4$\times$ the dimension of token embeddings.
They are thus likely to be well conditioned.
We therefore focus on widening the attention block by increasing the number of heads.
According to \cref{thm:main}, we expect this to bring the condition number of
each attention block towards $1$.
We will verify empirically in
\cref{sec:exps}
that this is indeed the case (\cref{fig:vitb_condition_heads}).



\paragraph{Trading depth and width.}
The literature discussed above suggests that depth and width have complementary roles for
the optimization of neural models.
We therefore hypothesize that increasing 
the number of attention heads
could be matched with a reduction in depth while maintaining performance.
The motivation stems from the fact that each layer uses a large amount of parameters,
hence a reduction in depth quickly decreases the model size.
In other words, additional attention heads could enable the design of compact transformers that perform comparably to deeper ones.

The experiments in \cref{sec:exps}
will extensively validate this hypothesis across a range of architectures and tasks.
A theoretical explanation as to why reducing depth yields such strong performance
is still incomplete. Our results open important questions for future work
about optimal architecture design from both theoretical and empirical perspectives.


\section{Experiments}\label{sec:exps}

We perform extensive experiments 
with a variety of transformer-based models.
Our goals are (1)~to empirically verify the prediction of \cref{thm:main} about improvements in conditioning
and (2)~to evaluate the downstream benefits on standard vision and NLP tasks: image classification with ImageNet-1k~\cite{steiner2021train}, language modeling with TinyStories~\cite{eldan2023tinystories}, and long-context reasoning with the LRA benchmark~\cite{tay2021long}.

\subsection{Image Classification}\label{subsec:image_classi}

We consider standard large vision transformers (ViTs) from the literature.
We modify their architecture according to the findings from \cref{sec:theory} and re-train them from scratch on ImageNet-1k~\cite{steiner2021train}.
Our approach enables reductions in parameter count by up to 30\%--50\% of existing models
without compromising their accuracy. 
The explicit training details, implementation and hardware used for all experiments in this subsection can be found in \cref{supp:subsec:vits}. 

\subsubsection{Standard ViTs}
We use the ViT-Base (ViT-B) architecture~\cite{dosovitskiy2020image},
a popular model for image classification.
The model processes an input image as non-overlapping patches of $16\!\times\!16$ pixels.
They are linearly projected into token embeddings of dimension $768$
that serve as input to the transformer layers.
ViT-B uses $12$ layers, each with $12$ attention heads of dimension $64$ ($12 \times 64\!=\!768$, the initial token embedding size).
Its MLPs use hidden layers of size $4 \times 768\!=\!3,072$.

\paragraph{Validating the effects on conditioning.}
To validate \cref{thm:main},
we systematically vary the number of heads in a ViT-B and re-train the model on ImageNet-1k.
We train each model to convergence i.e.\ for about 300 epochs.
For each training run, every 50 epochs, we compute the condition number of each layer's attention matrix and average them across layers. 
We examine the results in \cref{fig:vitb_condition_heads}
and observe that the condition number decreases markedly as the number of heads increases, thus validating \cref{thm:main}.

\begin{figure}[h!]
    \centering
    \includegraphics[width=0.92\linewidth]{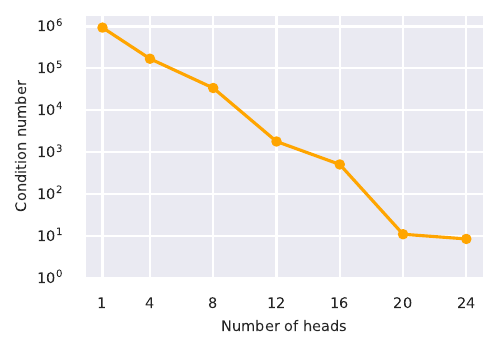}
    \caption{Empirical measurement of the condition number of the attention layers
    in ViT-Bs with different numbers of heads.
    The conditioning improves (lower number) with additional heads,
    following the predictions of \cref{thm:main}.}
    \label{fig:vitb_condition_heads}
    \vspace{-14pt}
\end{figure}

\begin{figure*}[t]
    \centering
    \footnotesize \textbf{\hspace{1.5em} 12 Layers \hspace{21em} 8 Layers}\\[-1pt]
    \includegraphics[width=0.81\linewidth,trim=0 0 0 5pt,clip]{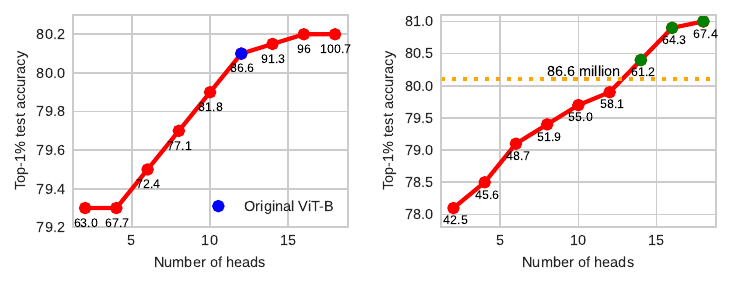}
    \vspace{-10pt}
    \caption{
    Accuracy on ImageNet-1k of variants of ViT-B with the original depth (12 layers, left)
    or reduced to 8 layers (right).
    Each point is annotated with the model's total number of parameters (in millions).
    According to our predictions, the number of heads correlates with performance.
    Remarkably, our models with reduced depth (right) and $\geq$12 heads (green dots)
    all obtain a \textbf{higher test accuracy with fewer parameters}
    than the original model (dotted line).}
    \label{fig:vitb_depth_heads}
\end{figure*}

\begin{figure*}[t]
    \vspace{-4pt}
    \centering
    \footnotesize \textbf{\hspace{0em} 12 Layers \hspace{21em} 8 Layers}\\[-1pt]
    \includegraphics[width=0.81\linewidth,trim=0 0 0 5pt,clip]{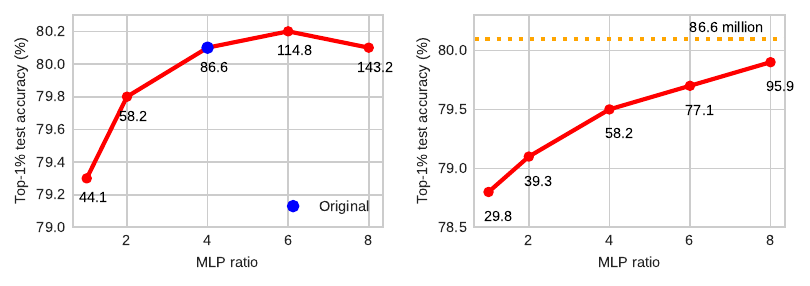}
    \vspace{-10pt}
    \caption{Similar experiments as \cref{fig:vitb_depth_heads},
    where each model is now a variant of ViT-B with a \textbf{different MLP width} (X axes, reported as a factor of the token-embedding size).
    According to our predictions, increasing the width of MLPs has a weaker effect than adding attention heads.
    The slight benefit observed with 12 layers (left) cannot compensate for a reduction of depth to 8 layers (right), unlike what was observed with additional heads in \cref{fig:vitb_depth_heads}.}
    \label{fig:vitb_mlp_ratio}
    \vspace{-4pt}
\end{figure*}

\paragraph{New model configurations.}
We first fix the depth at $12$ layers as in the original model, and vary the number of heads from $2$ to $18$, keeping a constant head dimension of $64$.
Following the discussion in \cref{subsec:heads_depth}, we then consider a reduced depth of $8$ layers, and vary again the number of heads from $2$ to $18$.
We train each configuration on ImageNet-1k and measure the top-1\% accuracy. Training uses the AdamW optimizer for $300$ epochs following standard strategy from prior work~\cite{steiner2021train} (details in \cref{app:expDetails}).

The results in \cref{fig:vitb_depth_heads} (left)
show a clear improvement in accuracy as the number of heads increases, including higher performance than the original model with $>$12 layers, at the cost of additional parameters.
We then examine a model with a depth reduced from 12 to 8 layers (\cref{fig:vitb_depth_heads}, right).
The accuracy is again correlated with the number of heads.
The smaller number of layers largely makes up for those in additional heads,
and all configurations with $>$12 heads surpass the accuracy of the original one
with a much smaller parameter count (61.2\,--\,67.4\,M vs.\ 86.6\,M).

\begin{figure}[h!]
    \centering
    {\footnotesize MLP hidden-layer width:}~
    \textcolor[HTML]{FFA301}{\raisebox{-0.03em}{\scalebox{0.9}[0.9]{\ding{110}}}}
    \footnotesize 2$\times$~~
    \textcolor[HTML]{FFA301}{\raisebox{-0.03em}{\scalebox{0.9}[0.9]{\ding{108}}}}
    \footnotesize 4$\times$ token embedding size
    \\\vspace{4pt}
    \includegraphics[width=0.99\linewidth]{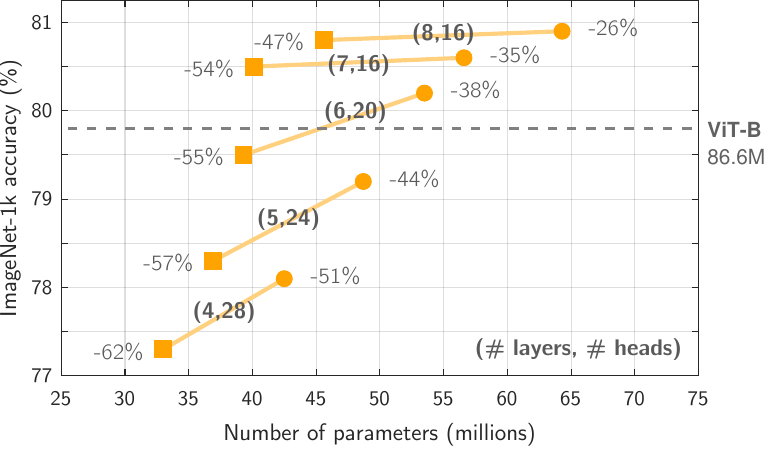}
    \caption{Additional variants of ViT-B with different numbers of layers and heads, and MLP width.
    Each model is annotated with its reduction in parameters.
    For 6–-8 layers, doubling the MLP width yields little benefit,
    indicating that the number of heads is more important.}
    \label{fig:vit_fig5}
    \vspace{-10pt}
\end{figure}

\begin{figure}[ht!]
    \centering
    ~~~~~\textcolor[HTML]{FFA301}{\raisebox{-0.03em}{\scalebox{0.9}[0.9]{\ding{108}}}}
    \footnotesize Ours (more heads, fewer layers)~~~
    \textcolor[HTML]{4D38A3}{\raisebox{-0.03em}{\scalebox{0.9}[0.9]{\ding{108}}}}
    \footnotesize Original
    \\\vspace{1pt}
    \includegraphics[width=0.813\linewidth]{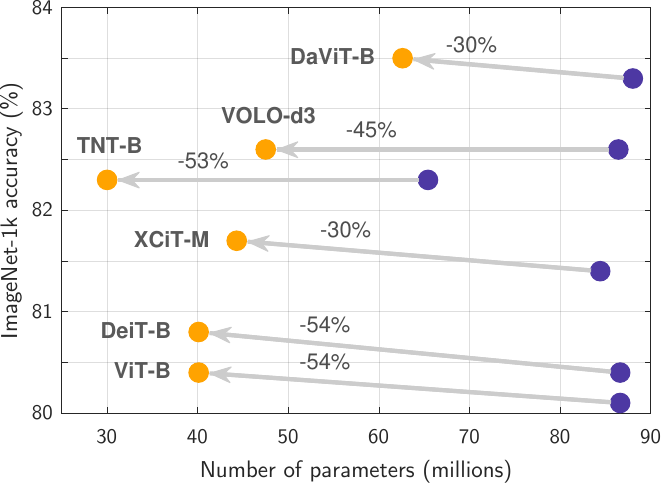}~~~~
    \\[1.9pt]
    \includegraphics[width=0.82\linewidth]{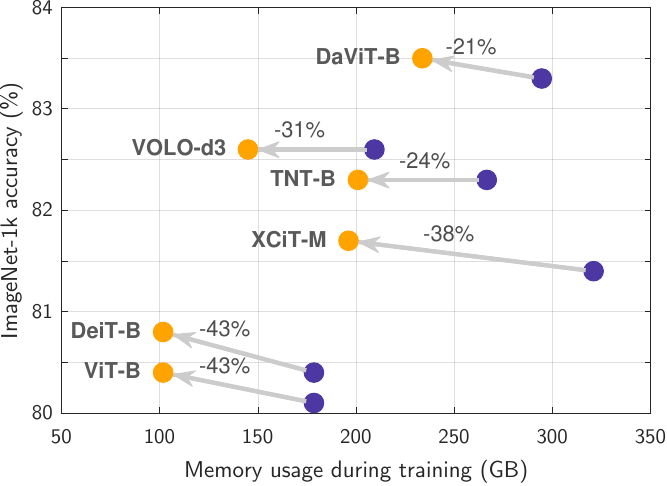}
    \vspace{-4pt}
    \caption{Other vision transformer architectures.
    We plot improvements in accuracy against reductions in parameter count (top) and memory usage during training (bottom).
    All models benefit significantly from our approach.}
    \label{fig:baseViTs}
\end{figure}

\begin{figure}[ht!]
    \centering
    ~~~~~\textcolor[HTML]{FFA301}{\raisebox{-0.03em}{\scalebox{0.9}[0.9]{\ding{108}}}}
    \footnotesize Ours (more heads, fewer layers)~~~
    \textcolor[HTML]{4D38A3}{\raisebox{-0.03em}{\scalebox{0.9}[0.9]{\ding{108}}}}
    \footnotesize Original
    \\\vspace{2.5pt}
    \includegraphics[width=0.82\linewidth]{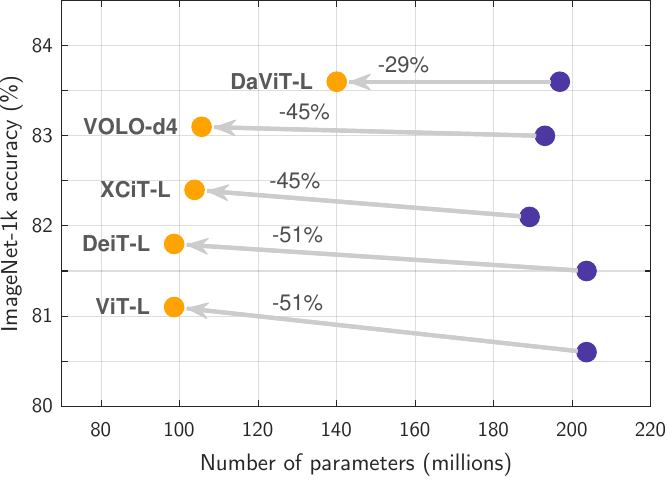}
    \\[3.5pt]
    \includegraphics[width=0.81\linewidth]{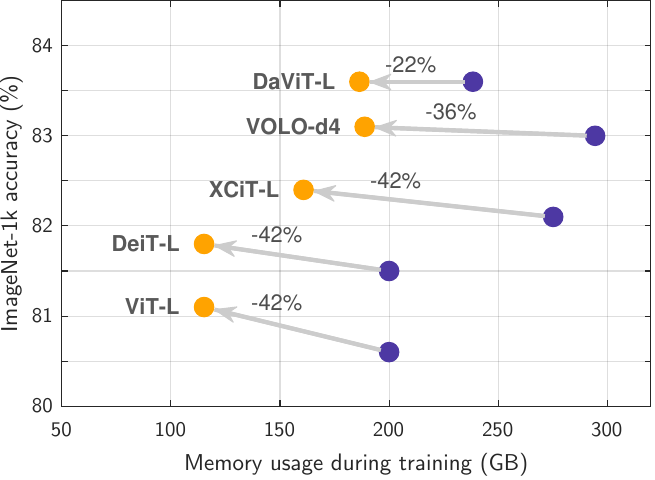}
    \vspace{-4pt}
    \caption{Large vision transformer architectures.
    We observe the same improvements in accuracy, parameter count, and memory usage as with other models.}
    \label{fig:largeViTs}
\end{figure}

\paragraph{MLP width.}
We now consider variations of the hidden-layer size of the MLPs inside a ViT-B model,
as an alternative strategy to affect the width of the model.
The original model uses a size of $768\!\times\!4\,=\,3,072$, where 768 is the token embedding size and 4 is referred to as the ``MLP ratio''.
We train models with a ratio between 1 and 8.
\cref{fig:vitb_mlp_ratio} shows a limited impact on accuracy
that contrasts with the clear large effects of the number of heads from \cref{fig:vitb_depth_heads}.
This agrees with the hypothesis made in \cref{subsec:heads_depth}
that MLPs are likely to be already well-conditioned and do not benefit in this regard as much as  attention blocks in transformers.

\paragraph{Best configurations.}
We evaluate additional configurations with depths below $8$ in \cref{fig:vit_fig5}.
We adjust the number of heads to match the accuracy of the original ViT-B ($\geq80.1\%$).
All configurations still use much fewer parameters than the original model
with a better accuracy.

\subsubsection{Other Vision Transformers}\label{subsubsec:other_vits}
We apply our strategy to a variety of alternative transformer-based architectures in the 60--90\,M parameter range:
DeiT \cite{touvron2021training}, XCiT \cite{ali2021xcit}, TNT \cite{han2021transformer}, VOLO \cite{yuan2022volo}, and DaViT \cite{ding2022davit}, all pretrained on ImageNet-1k.
We report our best configurations in \cref{fig:baseViTs} .
In all cases, reducing depth and increasing the number heads leads to models with similar or higher accuracy
with substantial reductions in parameter count.
This indicates that many models are unnecessarily oversized. 
This also corresponds to substantial reductions in memory during training (reported separately in \cref{fig:baseViTs}).

\vspace{-10pt}
\paragraph{Larger models.}
We also evaluate models in the 180\,--\,200\,M parameter range.
\cref{fig:largeViTs} shows similar improvements in accuracy, parameter count, and memory usage.

\subsection{Language Modeling}\label{subsec:language_models}
We evaluate our approach on two language models.

\vspace{-8pt}
\paragraph{Crammed BERT.}
We first consider the Crammed-BERT architecture~\cite{geiping2023cramming}.
trained on the Pile dataset~\cite{gao2021pile} following~\citet{geiping2023cramming}.
We evaluate these models on the GLUE benchmark~\cite{wang2018glue}.

\begin{table*}[h]
    \centering
    \small
    \setlength{\tabcolsep}{3pt}
    \renewcommand{\arraystretch}{1.15}
    \begin{tabular}{l|cccccccc|c|c|c}
        \toprule
        & MNLI & SST-2 & STSB & RTE & QNLI & QQP & MRPC & CoLA & GLUE & Parameters~~~ & Memory~~~ \\
        \midrule
        Crammed BERT {\footnotesize(original)}~~ & \textbf{83.8} & 92.3 & 86.3 & 55.1 & \textbf{90.1} & 87.3 & 85.0 & 48.9 & 78.6 
        & ~~~~$119$\,M~~~~~~~~~~~~  & ~~~~$13.8$\,GB~~~~~~~~~~~~ \\
        Crammed BERT {\footnotesize(ours)} & 83.7 & 92.3 &  86.3 & \textbf{55.3}  & 90.0  & 87.3  & \textbf{85.2}  & 48.9  & 78.6    
        &  ~~~~$84$\,M {\scriptsize($-29\%$)} & ~~~~$10.3$\,GB {\scriptsize($-25\%$)} \\
        \bottomrule
    \end{tabular}
    \caption{Comparison of a pretrained original Crammed BERT (16 layers, 12 heads per layer) with our leaner variant (10 layers, 24 heads) on the GLUE benchmark. For each task our learner variant achieves comparable performance with much less parameters.}
    \label{tab:bert_glue_results}
\end{table*}

We train several variants of Crammed BERT with different numbers of attention heads and layers.
The original model uses 12 heads and 16 layers.
As hypothesized, we find that increasing the number of heads leads to better performance, so much so that
the depth can be reduced and still match the performance of the original model (see \cref{tab:bert_glue_results}).
In particular, we find that 24 attention heads and 10 layers produce a compact architecture that performs similarly on GLUE as the original model.

\paragraph{GPT-2.}
We proceed similarly with a GPT-2 architecture
trained on the TinyStories dataset~\cite{eldan2023tinystories}.
As the original configuration, we use the 12-layer, 12-head model (89\,M parameters) from \citet{eldan2023tinystories}.
We then increase the number of heads to 16 while reducing the depth to 4 layers.
As shown in \cref{tab:gpt_results}, our variant outperforms the original one in validation loss. Moreover, it achieves these improvements with significantly fewer parameters and reduced memory usage during training.

%

\begin{table}[h]
    \centering
    \small
    \setlength{\tabcolsep}{3pt}
    \renewcommand{\arraystretch}{1.15}
    \begin{tabular}{lcc c}
        \toprule
         & Val. loss & Parameters & Memory\\ 
        \midrule
        GPT-2 {\footnotesize(original)} & 2.47 & $89$\,M & $12.8$\,GB \\

        GPT-2  {\footnotesize(ours)} & \textbf{2.41} & $64$\,M {\scriptsize(-$28\%$)} & $9.7$\,GB {\scriptsize(-$24\%$)}\\
        \bottomrule
    \end{tabular}
    \caption{GPT-2 models trained on the TinyStories dataset. We compare a baseline model with 12 layers and 12 attention heads~\cite{eldan2023tinystories} and our variant with 4 layers and 16 heads. We achieve superior performance at a much smaller size and memory usage.}
    \label{tab:gpt_results}
\end{table}

\subsection{LRA Benchmark with Nystr\"omformers}\label{subsec:nytstroem}

\begin{figure*}[t]
    \centering
    \footnotesize \textbf{\hspace{0em} 2 Layers \hspace{23.3em} 1 Layer}~\\
    \includegraphics[width=0.92\linewidth,trim=0 0 0 5pt,clip]
    {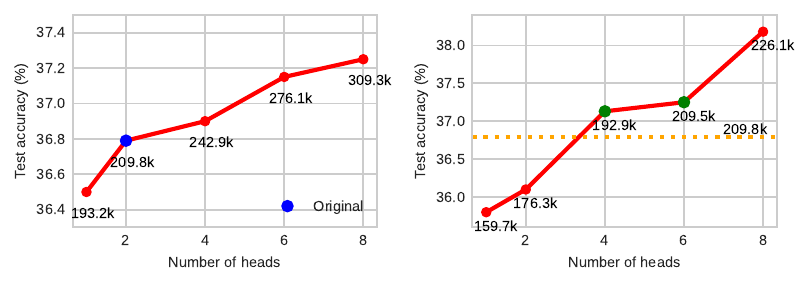}
    \vspace{-7pt}
    \caption{Accuracy on the ListOps task of the LRA benchmark with variants of the Nystr\"omformer.
    The original model from \citet{xiong2021nystromformer} uses 2 layers (left)
    and we also evaluate models with a single layers (right).
    Each model is annotated with its total number of parameters.
    According to our predictions, the number of heads correlates with performance.
    Remarkably, our models with just 1 layer and $\geq$\,4 heads (green dots)
    all obtain a \textbf{higher test accuracy with fewer parameters}
    than the original model (dotted line).}
    \label{fig:nystroem_heads}
    \vspace{9pt}
\end{figure*}

\begin{figure*}[t]
    \centering
    \footnotesize \textbf{\hspace{0em} 2 Layers \hspace{23.3em} 1 Layer}~\\
    \includegraphics[width=0.92\linewidth,trim=0 0 0 5pt,clip]
    {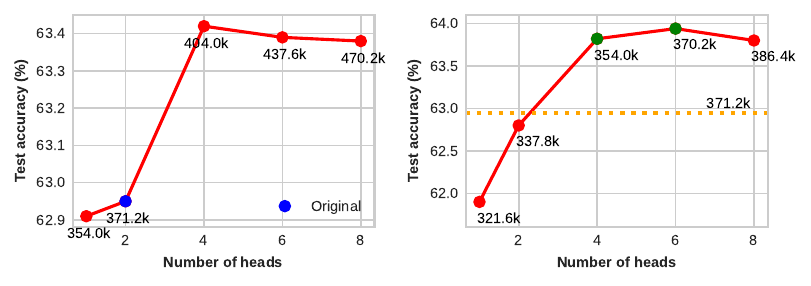}
    \vspace{-7pt}
    \caption{Accuracy on the text classification task of the LRA benchmark with variants of the Nystr\"omformer.
    The original model from \citet{xiong2021nystromformer} uses 2 layers (left)
    and we also evaluate models with a single layers (right).
    Each model is annotated with its total number of parameters.
    According to our predictions, the number of heads correlates with performance.
    Remarkably, our models with just 1 layer and $\geq$\,4 heads (green dots)
    all obtain a \textbf{higher test accuracy with fewer parameters}
    than the original model (dotted line).}
    \label{fig:nystroem_text_class}
\end{figure*}

We evaluate our approach on Nystr\"omformers~\cite{xiong2021nystromformer},
a transformer-like architecture that uses an approximation of the self-attention with better computational complexity.
Our objective is to evaluate the relevance of our findings to an architecture that slightly departs from the original transformer architecture of \citet{vaswani2017attention}.
Nystr\"omformers are well suited to long sequences and we therefore evaluate them on the
Long-Range Arena (LRA) benchmark~\cite{tay2021long}.

Our base model follows the original paper~\cite{xiong2021nystromformer}
and uses 2 layers and 2 attention heads per layer.
We also train variants with 2-8 heads and 1-2 layers.
The results on the ListOps task (see \cref{fig:nystroem_heads}) and the Text classification task (see \cref{fig:nystroem_text_class})
show that additional heads increase the accuracy.
This allows reducing the depth to a single layer while improving its accuracy.
These results hold across other tasks of the LRA benchmark (see \cref{tab:nystrom_results}).

\begin{table}[!ht]
    \centering
    \small
    \renewcommand{\arraystretch}{1.1}
    \setlength{\tabcolsep}{8pt}
    \begin{tabular}{cccl}
        \midrule
        \multicolumn{4}{c}{\normalsize\textbf{ListOps}} \\
        \midrule
        (Depth, heads) & Top-1\% Acc. & Parameters & ~~ \\
        \midrule
        (2, 2) & 36.79 & 209.8k &  \\[2pt]
        (1,4) & \textbf{37.13} & 192.9k & (-9\%) \\
        \midrule
        \multicolumn{4}{c}{\normalsize\textbf{Text Classification}} \\
        \midrule
        (Depth, heads) & Top-1\% Acc. & Parameters & ~~ \\
        \midrule
        (2, 2) & 62.95 & 371.2k &  \\[2pt]
        (1,4) & \textbf{63.82} & 354.0k & (-5\%) \\
        \midrule
        \multicolumn{4}{c}{\normalsize\textbf{Document Retrieval}} \\
        \midrule
        (Depth, heads) & Top-1\% Acc. & Parameters & ~~ \\
        \midrule
        (2, 2) & 79.3 & 394.8k &  \\[2pt]
        (1,4) & \textbf{79.5} & 394.8k & (same) \\
        \midrule
        \multicolumn{4}{c}{\normalsize\textbf{Image Classification}} \\
        \midrule
        (Depth, heads) & Top-1\% Acc. & Parameters & ~~ \\
        \midrule
        (2, 2) & 37.2 & 191.2k &  \\[2pt]
        (1,4) & \textbf{38.2} & 191.2k & (same) \\
        \midrule
        \multicolumn{4}{c}{\normalsize\textbf{Pathfinder}} \\
        \midrule
        (Depth, heads) & Top-1\% Acc. & Parameters & ~~ \\
        \midrule
        (2, 2) & 69.8 & 190.2k &  \\[2pt]
        (1,4) & \textbf{69.9} & 190.2k & (same) \\
        \midrule
    \end{tabular}
    \vspace{-2pt}
    \caption{Evaluation of variants of the Nystr\"omformer~\cite{xiong2021nystromformer} on different datasets of the Long-Range Arena (LRA) benchmark~\cite{tay2021long}.
    We compare the original model (2 layers, 2 heads) with our variant (1 layer, 4 heads). On every task, it outperforms the original model with the same number or slightly fewer parameters.}
    \label{tab:nystrom_results}
\end{table}


\section{Conclusions}

In this work, we reexamined the role of multi-head attention in transformers. Our theoretical analysis revealed that increasing the number of heads improves the conditioning of the attention matrices, a finding we confirmed empirically on vision transformers. Building on previous studies of MLP conditioning, we hypothesized that an increase of the number of heads could reduce the depth required to achieve high performance.
We tested this idea on tasks including image classification, language generation, and long sequence modeling, and found that leaner, shallower architectures with more attention heads perform comparably to their deeper counterparts. These results suggest a promising avenue for designing more efficient transformers without sacrificing performance.

\subsection*{Limitations and Open Questions}
\begin{itemize}[itemsep=6.5pt,topsep=3.5pt]
    \item We empirically demonstrated that depth can be traded off for more attention heads while maintaining performance. However, a theoretical explanation for this balance is still missing. Can we quantitatively predict the trade-offs of specific architectural variations?
    \vspace{16pt} 

    \item Our main theorem shows that increasing the number of heads improves the condition number of attention layers.
    The subsequent effect on task accuracy then rests on empirical results.
    How exactly does this form of conditioning impact training dynamics and downstream performance?
    
    \item
    Are there other architectural interventions that could achieve similar effects to the additional attention heads? Alternative methods for conditioning the attention layers could further improve the efficiency of transformers.
    
    \item Our resources allowed experiments on models with up to $\sim$200M parameters.
    Do the observed benefits persist at larger scales such as in $\sim$1B-parameter models?
\end{itemize}




\clearpage

{
    \small
    \bibliographystyle{ieeenat_fullname}
    \bibliography{main}

\begin{thebibliography}{45}
\providecommand{\natexlab}[1]{#1}
\providecommand{\url}[1]{\texttt{#1}}
\expandafter\ifx\csname urlstyle\endcsname\relax
  \providecommand{\doi}[1]{doi: #1}\else
  \providecommand{\doi}{doi: \begingroup \urlstyle{rm}\Url}\fi

\bibitem[Agarwal et~al.(2021)Agarwal, Awasthi, and Kale]{agarwal2021deep}
Naman Agarwal, Pranjal Awasthi, and Satyen Kale.
\newblock A deep conditioning treatment of neural networks.
\newblock In \emph{Algorithmic Learning Theory}, pages 249--305. PMLR, 2021.

\bibitem[Ali et~al.(2021)Ali, Touvron, Caron, Bojanowski, Douze, Joulin, Laptev, Neverova, Synnaeve, Verbeek, et~al.]{ali2021xcit}
Alaaeldin Ali, Hugo Touvron, Mathilde Caron, Piotr Bojanowski, Matthijs Douze, Armand Joulin, Ivan Laptev, Natalia Neverova, Gabriel Synnaeve, Jakob Verbeek, et~al.
\newblock Xcit: Cross-covariance image transformers.
\newblock \emph{Advances in neural information processing systems}, 34:\penalty0 20014--20027, 2021.

\bibitem[Arora et~al.(2018)Arora, Cohen, and Hazan]{arora2018optimization}
Sanjeev Arora, Nadav Cohen, and Elad Hazan.
\newblock On the optimization of deep networks: Implicit acceleration by overparameterization.
\newblock In \emph{International conference on machine learning}, pages 244--253. PMLR, 2018.

\bibitem[Carion et~al.(2020)Carion, Massa, Synnaeve, Usunier, Kirillov, and Zagoruyko]{carion2020end}
Nicolas Carion, Francisco Massa, Gabriel Synnaeve, Nicolas Usunier, Alexander Kirillov, and Sergey Zagoruyko.
\newblock End-to-end object detection with transformers.
\newblock In \emph{European conference on computer vision}, pages 213--229. Springer, 2020.

\bibitem[Choromanski et~al.(2021)Choromanski, Likhosherstov, Dohan, Song, Gane, Sarlos, Hawkins, Davis, Mohiuddin, Kaiser, et~al.]{choromanski2020rethinking}
Krzysztof Choromanski, Valerii Likhosherstov, David Dohan, Xingyou Song, Alex Gane, Tamas Sarlos, Peter Hawkins, Jared Davis, Afroz Mohiuddin, Lukasz Kaiser, et~al.
\newblock Rethinking attention with performers.
\newblock In \emph{International Conference on Learning Representations (ICLR)}, 2021.

\bibitem[Devlin et~al.(2018)Devlin, Chang, Lee, and Toutanova]{devlin2018bert}
Jacob Devlin, Ming-Wei Chang, Kenton Lee, and Kristina Toutanova.
\newblock Bert: Pre-training of deep bidirectional transformers for language understanding.
\newblock \emph{arXiv preprint arXiv:1810.04805}, 2018.

\bibitem[Ding et~al.(2022)Ding, Xiao, Codella, Luo, Wang, and Yuan]{ding2022davit}
Mingyu Ding, Bin Xiao, Noel Codella, Ping Luo, Jingdong Wang, and Lu Yuan.
\newblock Davit: Dual attention vision transformers.
\newblock In \emph{European conference on computer vision}, pages 74--92. Springer, 2022.

\bibitem[Dosovitskiy et~al.(2020)Dosovitskiy, Beyer, Kolesnikov, Weissenborn, Zhai, Unterthiner, Dehghani, Minderer, Heigold, Gelly, et~al.]{dosovitskiy2020image}
Alexey Dosovitskiy, Lucas Beyer, Alexander Kolesnikov, Dirk Weissenborn, Xiaohua Zhai, Thomas Unterthiner, Mostafa Dehghani, Matthias Minderer, Georg Heigold, Sylvain Gelly, et~al.
\newblock An image is worth 16x16 words: Transformers for image recognition at scale.
\newblock \emph{arXiv preprint arXiv:2010.11929}, 2020.

\bibitem[Eldan and Li(2023)]{eldan2023tinystories}
Ronen Eldan and Yuanzhi Li.
\newblock Tinystories: How small can language models be and still speak coherent english?
\newblock \emph{arXiv preprint arXiv:2305.07759}, 2023.

\bibitem[Fu et~al.(2024)Fu, Li, Wen, Dou, Cai, Shi, and Qiao]{fu2024drive}
Daocheng Fu, Xin Li, Licheng Wen, Min Dou, Pinlong Cai, Botian Shi, and Yu Qiao.
\newblock Drive like a human: Rethinking autonomous driving with large language models.
\newblock In \emph{Proceedings of the IEEE/CVF Winter Conference on Applications of Computer Vision}, pages 910--919, 2024.

\bibitem[Gao et~al.(2021)Gao, Biderman, Black, Golding, Hoppe, Foster, Phang, He, Thite, Nabeshima, et~al.]{gao2021pile}
Leo Gao, Stella Biderman, Sid Black, Laurence Golding, Travis Hoppe, Charles Foster, Jason Phang, Horace He, Aadi Thite, Eric Nabeshima, et~al.
\newblock The pile: An 800gb dataset of diverse text for language modeling.
\newblock \emph{arXiv preprint arXiv:2101.00027}, 2021.

\bibitem[Geiping(2023)]{CrammingGit}
Jonas Geiping.
\newblock Cramming.
\newblock \url{https://github.com/JonasGeiping/cramming}, 2023.

\bibitem[Geiping and Goldstein(2023)]{geiping2023cramming}
Jonas Geiping and Tom Goldstein.
\newblock Cramming: Training a language model on a single gpu in one day.
\newblock In \emph{International Conference on Machine Learning}, pages 11117--11143. PMLR, 2023.

\bibitem[Han et~al.(2021)Han, Xiao, Wu, Guo, Xu, and Wang]{han2021transformer}
Kai Han, An Xiao, Enhua Wu, Jianyuan Guo, Chunjing Xu, and Yunhe Wang.
\newblock Transformer in transformer.
\newblock \emph{Advances in neural information processing systems}, 34:\penalty0 15908--15919, 2021.

\bibitem[Jacot et~al.(2018)Jacot, Gabriel, and Hongler]{jacot2018neural}
Arthur Jacot, Franck Gabriel, and Cl{\'e}ment Hongler.
\newblock Neural tangent kernel: Convergence and generalization in neural networks.
\newblock \emph{Advances in neural information processing systems}, 31, 2018.

\bibitem[Kabkab et~al.(2016)Kabkab, Hand, and Chellappa]{kabkab2016size}
Maya Kabkab, Emily Hand, and Rama Chellappa.
\newblock On the size of convolutional neural networks and generalization performance.
\newblock In \emph{2016 23rd International Conference on Pattern Recognition (ICPR)}, pages 3572--3577. IEEE, 2016.

\bibitem[Kitaev et~al.(2020)Kitaev, Kaiser, and Levskaya]{kitaev2020reformer}
Nikita Kitaev, Łukasz Kaiser, and Anselm Levskaya.
\newblock Reformer: The efficient transformer.
\newblock In \emph{International Conference on Learning Representations (ICLR)}, 2020.

\bibitem[Levine et~al.(2020{\natexlab{a}})Levine, Wies, Sharir, Bata, and Shashua]{levine2020depth}
Yoav Levine, Noam Wies, Or Sharir, Hofit Bata, and Amnon Shashua.
\newblock The depth-to-width interplay in self-attention.
\newblock \emph{arXiv preprint arXiv:2006.12467}, 2020{\natexlab{a}}.

\bibitem[Levine et~al.(2020{\natexlab{b}})Levine, Wies, Sharir, Bata, and Shashua]{levine2020limits}
Yoav Levine, Noam Wies, Or Sharir, Hofit Bata, and Amnon Shashua.
\newblock Limits to depth efficiencies of self-attention.
\newblock \emph{NeurIPS}, 33:\penalty0 22640--22651, 2020{\natexlab{b}}.

\bibitem[Li et~al.(2018)Li, Lu, Wang, Haupt, and Zhao]{li2018tighter}
Xingguo Li, Junwei Lu, Zhaoran Wang, Jarvis Haupt, and Tuo Zhao.
\newblock On tighter generalization bound for deep neural networks: Cnns, resnets, and beyond.
\newblock \emph{arXiv preprint arXiv:1806.05159}, 2018.

\bibitem[Liu et~al.(2022)Liu, Zhu, and Belkin]{liu2022loss}
Chaoyue Liu, Libin Zhu, and Mikhail Belkin.
\newblock Loss landscapes and optimization in over-parameterized non-linear systems and neural networks.
\newblock \emph{Applied and Computational Harmonic Analysis}, 59:\penalty0 85--116, 2022.

\bibitem[Liu et~al.(2021)Liu, Lin, Cao, Hu, Wei, Zhang, Lin, and Guo]{liu2021swin}
Ze Liu, Yutong Lin, Yue Cao, Han Hu, Yixuan Wei, Zheng Zhang, Stephen Lin, and Baining Guo.
\newblock Swin transformer: Hierarchical vision transformer using shifted windows.
\newblock In \emph{Proceedings of the IEEE/CVF international conference on computer vision}, pages 10012--10022, 2021.

\bibitem[Lu et~al.(2017)Lu, Pu, Wang, Hu, and Wang]{lu2017expressive}
Zhou Lu, Hongming Pu, Feicheng Wang, Zhiqiang Hu, and Liwei Wang.
\newblock The expressive power of neural networks: A view from the width.
\newblock \emph{Advances in neural information processing systems}, 30, 2017.

\bibitem[Maiti et~al.(2023)Maiti, Elberink, and Vosselman]{maiti2023transfusion}
Abhisek Maiti, Sander~Oude Elberink, and George Vosselman.
\newblock Transfusion: Multi-modal fusion network for semantic segmentation.
\newblock In \emph{Proceedings of the IEEE/CVF Conference on Computer Vision and Pattern Recognition}, pages 6536--6546, 2023.

\bibitem[Nocedal and Wright(1999)]{nocedal1999numerical}
Jorge Nocedal and Stephen~J Wright.
\newblock \emph{Numerical optimization}.
\newblock Springer, 1999.

\bibitem[Petty et~al.(2023)Petty, van Steenkiste, Sha, Dasgupta, Garrette, and Linzen]{petty2023impact}
Jackson Petty, Sjoerd van Steenkiste, Fei Sha, Ishita Dasgupta, Dan Garrette, and Tal Linzen.
\newblock The impact of depth and width on transformer language model generalization.
\newblock \emph{openreview}, 2023.

\bibitem[Poole et~al.(2016)Poole, Lahiri, Raghu, Sohl-Dickstein, and Ganguli]{poole2016exponential}
Ben Poole, Subhaneil Lahiri, Maithra Raghu, Jascha Sohl-Dickstein, and Surya Ganguli.
\newblock Exponential expressivity in deep neural networks through transient chaos.
\newblock \emph{Advances in neural information processing systems}, 29, 2016.

\bibitem[Raja()]{tinystories_git}
Praveen Raja.
\newblock Tiny-stories-gpt.
\newblock \url{https://github.com/PraveenRaja42/Tiny-Stories-GPT}.

\bibitem[Salzmann et~al.(2020)Salzmann, Ivanovic, Chakravarty, and Pavone]{salzmann2020trajectron++}
Tim Salzmann, Boris Ivanovic, Punarjay Chakravarty, and Marco Pavone.
\newblock Trajectron++: Multi-agent generative trajectory forecasting with heterogeneous data for control.
\newblock \emph{arXiv preprint arXiv:2001.03093}, 2, 2020.

\bibitem[Sanford et~al.(2023)Sanford, Hsu, and Telgarsky]{sanford2023representational}
Clayton Sanford, Daniel~J Hsu, and Matus Telgarsky.
\newblock Representational strengths and limitations of transformers.
\newblock \emph{NeurIPS}, 36:\penalty0 36677--36707, 2023.

\bibitem[Stein and Shakarchi(2009)]{stein2009real}
Elias~M Stein and Rami Shakarchi.
\newblock \emph{Real analysis: measure theory, integration, and Hilbert spaces}.
\newblock Princeton University Press, 2009.

\bibitem[Steiner et~al.(2021)Steiner, Kolesnikov, Zhai, Wightman, Uszkoreit, and Beyer]{steiner2021train}
Andreas Steiner, Alexander Kolesnikov, Xiaohua Zhai, Ross Wightman, Jakob Uszkoreit, and Lucas Beyer.
\newblock How to train your vit? data, augmentation, and regularization in vision transformers.
\newblock \emph{arXiv preprint arXiv:2106.10270}, 2021.

\bibitem[Tay et~al.(2021)Tay, Dehghani, Abnar, Shen, Bahri, Pham, Rao, Heinrich, Hua, and Metzler]{tay2021long}
Yi Tay, Mostafa Dehghani, Samira Abnar, Yikang Shen, Dara Bahri, Philip Pham, Jinfeng Rao, Julian Heinrich, Dai Hua, and Donald Metzler.
\newblock Long range arena: A benchmark for efficient transformers.
\newblock \emph{arXiv preprint arXiv:2011.04006}, 2021.

\bibitem[Touvron et~al.(2021)Touvron, Cord, Douze, Massa, Sablayrolles, and J{\'e}gou]{touvron2021training}
Hugo Touvron, Matthieu Cord, Matthijs Douze, Francisco Massa, Alexandre Sablayrolles, and Herv{\'e} J{\'e}gou.
\newblock Training data-efficient image transformers \& distillation through attention.
\newblock In \emph{International conference on machine learning}, pages 10347--10357. PMLR, 2021.

\bibitem[Vardi et~al.(2022)Vardi, Yehudai, and Shamir]{vardi2022width}
Gal Vardi, Gilad Yehudai, and Ohad Shamir.
\newblock Width is less important than depth in relu neural networks.
\newblock In \emph{Conference on learning theory}, pages 1249--1281. PMLR, 2022.

\bibitem[Vaswani(2017)]{vaswani2017attention}
A Vaswani.
\newblock Attention is all you need.
\newblock \emph{Advances in Neural Information Processing Systems}, 2017.

\bibitem[Vershynin(2018)]{vershynin2018high}
Roman Vershynin.
\newblock \emph{High-dimensional probability: An introduction with applications in data science}.
\newblock Cambridge university press, 2018.

\bibitem[Wang et~al.(2018)Wang, Singh, Michael, Hill, Levy, and Bowman]{wang2018glue}
Alex Wang, Amanpreet Singh, Julian Michael, Felix Hill, Omer Levy, and Samuel~R Bowman.
\newblock Glue: A multi-task benchmark and analysis platform for natural language understanding.
\newblock \emph{arXiv preprint arXiv:1804.07461}, 2018.

\bibitem[Wang et~al.(2020)Wang, Li, Khabsa, Fang, and Ma]{wang2020linformer}
Sinong Wang, Belinda~Z Li, Madian Khabsa, Han Fang, and Hao Ma.
\newblock Linformer: Self-attention with linear complexity.
\newblock In \emph{Advances in Neural Information Processing Systems (NeurIPS)}, 2020.

\bibitem[Xiong et~al.(2021{\natexlab{a}})Xiong, Zeng, Chakraborty, Tan, Fung, Singh, Yuan, Wang, Papailiopoulos, and Fragkiadaki]{nystromformer_github}
Yunyang Xiong, Zhanpeng Zeng, Rudrasis Chakraborty, Fei Tan, Glenn Fung, Vikas Singh, Xiaodong Yuan, Sungsoo~Ahn Wang, Dimitris Papailiopoulos, and Katerina Fragkiadaki.
\newblock Github repository, 2021{\natexlab{a}}.

\bibitem[Xiong et~al.(2021{\natexlab{b}})Xiong, Zeng, Chakraborty, Tan, Fung, Li, and Singh]{xiong2021nystromformer}
Yunyang Xiong, Zhanpeng Zeng, Rudrasis Chakraborty, Mingxing Tan, Glenn Fung, Yin Li, and Vikas Singh.
\newblock Nystr{\"o}mformer: A nystr{\"o}m-based algorithm for approximating self-attention.
\newblock \emph{Proceedings of the AAAI Conference on Artificial Intelligence}, 2021{\natexlab{b}}.

\bibitem[Yuan et~al.(2022)Yuan, Hou, Jiang, Feng, and Yan]{yuan2022volo}
Li Yuan, Qibin Hou, Zihang Jiang, Jiashi Feng, and Shuicheng Yan.
\newblock Volo: Vision outlooker for visual recognition.
\newblock \emph{IEEE transactions on pattern analysis and machine intelligence}, 45\penalty0 (5):\penalty0 6575--6586, 2022.

\bibitem[Zhen et~al.(2022)Zhen, Sun, Deng, Li, Wei, Lv, Yan, Kong, and Zhong]{zhen2022cosformer}
Q Zhen, W Sun, H Deng, D Li, Y Wei, B Lv, J Yan, L Kong, and Y Zhong.
\newblock cosformer: rethinking softmax in attention.
\newblock In \emph{International Conference on Learning Representations}, 2022.

\bibitem[Zhou and Feng(2018)]{zhou2018understanding}
Pan Zhou and Jiashi Feng.
\newblock Understanding generalization and optimization performance of deep cnns.
\newblock In \emph{International Conference on Machine Learning}, pages 5960--5969. PMLR, 2018.

\bibitem[Zhuang et~al.(2021)Zhuang, Wayne, Ya, and Jun]{zhuang2021robustly}
Liu Zhuang, Lin Wayne, Shi Ya, and Zhao Jun.
\newblock A robustly optimized bert pre-training approach with post-training.
\newblock In \emph{Proceedings of the 20th chinese national conference on computational linguistics}, pages 1218--1227, 2021.

\end{thebibliography}
}

\clearpage
\setcounter{page}{1}
\appendix
\onecolumn
\maketitlesupplementary

\section{Theoretical Framework}\label{app:theory}

In \cref{sec:theory} we used Lemma \ref{lem;random_draws} in the proof of our main \cref{thm:main}. We give the proof of the lemma.

\begin{proof}[Proof of Lemma \ref{lem;random_draws}]
We first note that any measure defined via a Gaussian or probability distribution is absolutely continuous with respect to the Lebesgue measure \cite{stein2009real}. Meaning they have the same sets of measure zero as the Lebesgue measure.

Write $X = [X_1,\ldots, X_n]$ where each $X_i \in \R^m$ for 
$1 \leq i \leq n$. We first prove the case that 
that $\{X_1,\ldots,X_n\}$ are vectors of unit length. Since the vectors were drawn independently, we can first assume we drew $X_1$. The probability that this is the zero vector is $0$ w.r.t the Lebesgue measure on the closed unit ball $B_N(0)$ about the origin in $\R^N$ and hence any other measure absolutely continuous to it. Then draw $X_2$ and note that the probability that $X_2$ lies in 
$span\{X_1\} \cap B_N(0)$ is also $0$ since $span\{X_1\} \cap B_N(0)$ forms a set of $0$ Lebesgue measure in $B_N(0)$. Continuing in this way we find that $\{X_1,\ldots,X_n\}$ will be linearly independent with probability $1$ implying that the matrix $X$ has full rank.

For the general case where $\{X_1,\ldots,X_n\}$ are not drawn to have unit length i.e. drawn on the sphere in $\R^N$, we simply note that we can draw each one and then divide by its norm producing one of unit length. Since normalizing by the norm doesn't affect linear independence we get by the above case that  $\{X_1,\ldots,X_n\}$ must be linearly independent with probability~$1$.    
\end{proof}

\section{Experimental Details}
\label{app:expDetails}

\subsection{Vision transformers on ImageNet-1k}\label{supp:subsec:vits}

\paragraph{Detailed results for vision transformers} In \cref{subsubsec:other_vits}, we demonstrated that several base vision transformers from the literature, ranging from 60 to 90 million parameters, benefit from our approach of increasing the number of heads in each attention layer while reducing the overall depth. In every instance, our configuration performed on par with or better than the original architecture while significantly lowering both parameter count and memory usage (see \cref{fig:baseViTs}). The detailed configurations are provided in \cref{supp:tab:vitbs}. 

We also showed that our methodology could be applied to larger vision transformers with roughly 180-200 million parameters (\cref{fig:largeViTs}). The configurations for these larger ViTs are given in \cref{supp:tab:vitls}.

\begin{table*}[!ht]
    \centering
    \resizebox{\textwidth}{!}{%
    \setlength{\tabcolsep}{12pt}
    \begin{tabular}{c|c|c c c c}
        \midrule
        \multicolumn{6}{c}{\textbf{ViT-B on ImageNet-1k}} \\
        \midrule
        (Depth, Heads) & MLP dim. & Top-1\% Acc. & Top-5\% Acc. & Params. (millions) & Memory (GB) \\
        \midrule
        \textcolor{red}{(12, 12)} & \textcolor{red}{3072} & \textcolor{red}{80.1} & \textcolor{red}{94.2} & \textcolor{red}{86.6} & \textcolor{red}{178.4} \\
        \textcolor{green}{(7,16)} & \textcolor{green}{1536} & \textcolor{green}{80.4} & \textcolor{green}{94.9} & \textcolor{green}{40.1} & \textcolor{green}{101.6} \\
        \midrule
        \multicolumn{6}{c}{\textbf{DeiT-B on ImageNet-1k}} \\
        \midrule
        (Depth, Heads) & MLP dim. & Top-1\% Acc. & Top-5\% Acc. & Params. (millions) & Memory (GB) \\
        \midrule
        \textcolor{red}{(12, 12)} & \textcolor{red}{3072} & \textcolor{red}{80.4} & \textcolor{red}{95.1} & \textcolor{red}{86.6} & \textcolor{red}{178.4} \\
        \textcolor{green}{(7,16)} & \textcolor{green}{1536} & \textcolor{green}{80.8} & \textcolor{green}{95.3} & \textcolor{green}{40.1} & \textcolor{green}{101.6} $\color{green} \downarrow$ \\
        \midrule
        \multicolumn{6}{c}{\textbf{XCiT-Medium on ImageNet-1k}} \\
        \midrule
        (Depth, Heads) & MLP dim. & Top-1\% Acc. & Top-5\% Acc. & Params. (millions) & Memory (GB) \\
        \midrule
        \textcolor{red}{(24, 8)} & \textcolor{red}{2048} & \textcolor{red}{81.4} & \textcolor{red}{95.5} & \textcolor{red}{84.4} & \textcolor{red}{320.8} \\
        \textcolor{green}{(12,16)} & \textcolor{green}{2048} & \textcolor{green}{81.7} & \textcolor{green}{95.6} & \textcolor{green}{59.0} & \textcolor{green}{196} \\
        \midrule
        \multicolumn{6}{c}{\textbf{TNT-B on ImageNet-1k}} \\
        \midrule
        (Depth, Heads) & MLP dim. & Top-1\% Acc. & Top-5\% Acc. & Params. (millions) & Memory (GB) \\
        \midrule
        \textcolor{red}{(12, 10)} & \textcolor{red}{2560} & \textcolor{red}{82.3} & \textcolor{red}{95.7} & \textcolor{red}{65.4} & \textcolor{red}{266.4} \\
        \textcolor{green}{(8,16)} & \textcolor{green}{2560} & \textcolor{green}{82.3} & \textcolor{green}{95.8} & \textcolor{green}{30.9} & \textcolor{green}{200.8} \\
        \midrule
        \multicolumn{6}{c}{\textbf{VOLO-d3 on ImageNet-1k}} \\
        \midrule
        (Depth, Heads) & MLP dim. & Top-1\% Acc. & Top-5\% Acc. & Params. (millions) & Memory (GB) \\
        \midrule
        \textcolor{red}{([8, 8, 16, 4], [8, 16, 16, 16])} & \textcolor{red}{(1024, 2048, 2048, 2048)} & \textcolor{red}{82.6} & \textcolor{red}{95.6} & \textcolor{red}{86} & \textcolor{red}{209.2} \\
        \textcolor{green}{([4, 4, 8, 2], [16, 32, 32, 32])} & \textcolor{green}{(768, 1536, 1536, 1536)} & \textcolor{green}{82.6} & \textcolor{green}{95.7} & \textcolor{green}{47.5} & \textcolor{green}{144.8} \\
        \midrule
        \multicolumn{6}{c}{\textbf{DaViT-B on ImageNet-1k}} \\
        \midrule
        (Depth, Heads) & MLP dim. & Top-1\% Acc. & Top-5\% Acc. & Params. (millions) & Memory (GB) \\
        \midrule
        \textcolor{red}{([1,1,9,1], [4, 8, 16, 32])} & \textcolor{red}{(512, 1024, 2048, 4096)} & \textcolor{red}{83.3} & \textcolor{red}{96.0} & \textcolor{red}{88.0} & \textcolor{red}{294.4} \\
        \textcolor{green}{([1, 1, 5, 1], [4, 8, 32, 32])} & \textcolor{green}{(512, 1024, 2048, 4096)} & \textcolor{green}{83.5} & \textcolor{green}{96.1} & \textcolor{green}{62.0} & \textcolor{green}{233.6} \\
        \midrule
        \end{tabular}%
    }
    \caption{Detailed configurations for a variety of base vision transformers from the literature. Increasing the heads and reducing depth (green) we obtain several transformers that outperform their original counterparts (red) with less parameters and less memory for training.}
    \label{supp:tab:vitbs}
\end{table*}

\begin{table*}[!ht]
    \centering
    \resizebox{\textwidth}{!}{%
    \setlength{\tabcolsep}{12pt}
    \begin{tabular}{c|c|c c c c}
        \midrule
        \multicolumn{6}{c}{\textbf{ViT-L on ImageNet-1k}} \\
        \midrule
        (Depth, Heads) & MLP dim. & Top-1\% Acc. & Top-5\% Acc. & Params. (millions) & Memory (GB) \\
        \midrule
        \textcolor{red}{(24, 16)} & \textcolor{red}{4096} & \textcolor{red}{80.6} & \textcolor{red}{94.4} & \textcolor{red}{203.6} & \textcolor{red}{200.0} \\
        \textcolor{green}{(8,30)} & \textcolor{green}{2048} & \textcolor{green}{81.1} & \textcolor{green}{95.1} & \textcolor{green}{98.6} & \textcolor{green}{115.2} \\
        \midrule
        \multicolumn{6}{c}{\textbf{DeiT-L on ImageNet-1k}} \\
        \midrule
        (Depth, Heads) & MLP dim. & Top-1\% Acc. & Top-5\% Acc. & Params. (millions) & Memory (GB) \\
        \midrule
        \textcolor{red}{(24, 16)} & \textcolor{red}{4096} & \textcolor{red}{81.5} & \textcolor{red}{95.3} & \textcolor{red}{203.6} & \textcolor{red}{200.0} \\
        \textcolor{green}{(8,30)} & \textcolor{green}{2048} & \textcolor{green}{81.8} & \textcolor{green}{95.4} & \textcolor{green}{98.6} & \textcolor{green}{115.2} $\color{green} \downarrow$ \\
        \midrule
        \multicolumn{6}{c}{\textbf{XCiT-L on ImageNet-1k}} \\
        \midrule
        (Depth, Heads) & MLP dim. & Top-1\% Acc. & Top-5\% Acc. & Params. (millions) & Memory (GB) \\
        \midrule
        \textcolor{red}{(24, 16)} & \textcolor{red}{3072} & \textcolor{red}{82.1} & \textcolor{red}{95.9} & \textcolor{red}{189.1} & \textcolor{red}{275.2} \\
        \textcolor{green}{(12,24)} & \textcolor{green}{3072} & \textcolor{green}{82.4} & \textcolor{green}{95.9} & \textcolor{green}{103.8} & \textcolor{green}{160.8} \\
        \midrule
        \multicolumn{6}{c}{\textbf{VOLO-d4 on ImageNet-1k}} \\
        \midrule
        (Depth, Heads) & MLP dim. & Top-1\% Acc. & Top-5\% Acc. & Params. (millions) & Memory (GB) \\
        \midrule
        \textcolor{red}{([8, 8, 16, 4], [12, 16, 16, 16])} & \textcolor{red}{(1536, 3072, 3072, 3072)} & \textcolor{red}{83.0} & \textcolor{red}{96.1} & \textcolor{red}{193.0} & \textcolor{red}{294.4} \\
        \textcolor{green}{([4, 4, 8, 2], [24, 32, 32, 32])} & \textcolor{green}{(768, 1536, 1536, 1536)} & \textcolor{green}{83.1} & \textcolor{green}{96.2} & \textcolor{green}{105.6} & \textcolor{green}{188.8} \\
        \midrule
        \multicolumn{6}{c}{\textbf{DaViT-L on ImageNet-1k}} \\
        \midrule
        (Depth, Heads) & MLP dim. & Top-1\% Acc. & Top-5\% Acc. & Params. (millions) & Memory (GB) \\
        \midrule
        \textcolor{red}{([1,1,9,1], [6, 12, 24, 48])} & \textcolor{red}{(768, 1536, 3072, 6144)} & \textcolor{red}{83.6} & \textcolor{red}{96.5} & \textcolor{red}{196.8} & \textcolor{red}{238.4} \\
        \textcolor{green}{([1, 1, 5, 1], [6, 12, 48, 48])} & \textcolor{green}{(768, 1536, 3072, 6144)} & \textcolor{green}{83.6} & \textcolor{green}{96.6} & \textcolor{green}{140.0} & \textcolor{green}{186.4} \\
        \midrule
        \end{tabular}%
    }
    \caption{Detailed configurations for a variety of large vision transformers from the literature. Increasing the heads and reducing depth (green) we obtain several transformers that outperform their original counterparts (red) with less parameters and less memory for training.}
    \label{supp:tab:vitls}
\end{table*}

\paragraph{Hardware and implementation.} 
 All models were trained on 8 Nvidia A100 GPUs using the code base from huggingface: https://github.com/huggingface/pytorch-image-models. Note that we couldn't find an implementation of a TNT large architecture in this code base and that is why we did not have TNT large in our analysis for large vision transformers. 
The training of each vision transformer architecture we considered follows \cite{steiner2021train} with explicit hyperparameter choices given in \cref{supp:tab:vit_hyperparameters}.

\begin{table}[!ht]
\centering
\small
\begin{tabular}{ll}
\toprule
\textbf{Hyperparameter} & \textbf{Value} \\
\midrule
Batch size & 1024 for base and 512 for large \\
\midrule
Number of epochs & 300 \\
\midrule
Learning rate & 3.00e-03 \\
\midrule
Optimizer & AdamW \\
\midrule
Weight decay & 0.3 \\
\midrule
Label smoothing & 0.1 \\
\midrule
Number of warm-up epochs & 20 \\
\midrule
Warmup learning rate & 1.00e-05 \\
\midrule
Mixup & 0.8 \\
\midrule
Cutmix & 1 \\
\midrule
Drop path & 0.1 \\
\midrule
RandAug & 9, 0.5 \\
\bottomrule
\end{tabular}
\caption{Hyperparameter settings for all vision transformer models.}
\label{supp:tab:vit_hyperparameters}
\end{table}

\subsection{Language Models}\label{supp:subsec:lm}

\paragraph{Hardware and implementation.} Both the Crammed BERT and GPT-2 models from \cref{subsec:language_models} were trained on one Nvidia A6000 GPU. The implementation, training and hyperparameters of the Crammed BERT model followed the original GitHub repo \cite{CrammingGit}. The GPT-2 models were trained following the paper \cite{eldan2023tinystories} and the github repo
\cite{tinystories_git}.

\subsection{Nystr\"omformer}

\paragraph{Hardware and implementation.} The Nystr\"omformer experiments carried out in \cref{subsec:nytstroem} were done on one Nvidia A6000 GPU. The implementation followed the original paper \cite{xiong2021nystromformer} and its GitHub repo \cite{nystromformer_github}.

\end{document}